\documentclass[runningheads]{llncs}
\usepackage[T1]{fontenc}
\usepackage{amsmath,amssymb}
\usepackage{thmtools}
\usepackage[hidelinks]{hyperref}

\makeatletter
\def\Bign#1{\mathclose{\hbox{$\left#1\vbox to9\p@{}\right.\n@space$}}\mathopen{}}
\makeatother

    \usepackage{mathrsfs}
    \usepackage{multirow}
    \usepackage{makecell}

  \newcommand{\abs}[1]{\vert #1 \vert}
  \newcommand{\angles}[1]{\langle #1 \rangle}

  \usepackage{bbm}

  \newcommand*\bell{\ensuremath{\boldsymbol\ell}}
  \newcommand*\bu{\ensuremath{\boldsymbol{u}}}
  \newcommand*\bm{\ensuremath{\boldsymbol{m}}}
  \newcommand*\bn{\ensuremath{\boldsymbol{n}}}

  \newcommand{\infball}{\mathcal{B}^{\infty}}

  \newcommand{\universe}{\mathbb{F}}

    \newcommand{\intervals}[1]{\mathbb{I}(#1)}

    \newcommand{\intervalsdxu}{\mathbb{I}(d)\vert^{\universe}_{\mathbf{x}}}
    
    \newcommand{\intervalsdzerou}{\mathbb{I}(d)\vert^{\universe}_{\mathbf{0}}}

  \newcommand{\vol}{v}

  \newcommand{\din}{{d_{\text{in}}}}
  \newcommand{\dout}{d_{\text{out}}}

  \newcommand{\uradi}{\underline{r}_i}
  \newcommand{\oradi}{\overline{r}_i}

  \usepackage{pifont}
  \usepackage[ruled, vlined, linesnumbered]{algorithm2e}
  \SetKwInput{KwInput}{Input}
  \SetKwInput{KwOutput}{Output} 

  \usepackage{adjustbox}
  
  \usepackage{booktabs}

  \usepackage{multicol}

  \usepackage{tikz}
  \usetikzlibrary{positioning,calc,arrows,shapes,fit}
  \usepackage{pgfplots}
  \pgfplotsset{width=10cm,compat=1.9}
  \usepgfplotslibrary{fillbetween}
  \usetikzlibrary{arrows.meta}
  \usetikzlibrary{backgrounds,3d}

  \usepackage{setspace}

  \usepackage{tkz-euclide,tikz,picture,xcolor,color,tikz-cd}
  \usetikzlibrary{arrows,chains}
  \usetikzlibrary{automata,positioning,shapes,calc,patterns}
  \usetikzlibrary{matrix,tikzmark}
  \usetikzlibrary{shapes.geometric}
  \usetikzlibrary{patterns.meta,graphs,quotes,shadows}

  \usepackage{pgfplots}
  \usepgfplotslibrary{groupplots}
  \pgfplotsset{compat=1.17}
  \usetikzlibrary{patterns}

  \usepackage{graphicx}
  \usepackage{caption}
  \usepackage{subcaption}

  \tikzset{
      ncbar angle/.initial=90,
      ncbar/.style={
          to path=(\tikztostart)
          -- ($(\tikztostart)!#1!\pgfkeysvalueof{/tikz/ncbar angle}:(\tikztotarget)$)
          -- ($(\tikztotarget)!($(\tikztostart)!#1!\pgfkeysvalueof{/tikz/ncbar angle}:(\tikztotarget)$)!\pgfkeysvalueof{/tikz/ncbar angle}:(\tikztostart)$)
          -- (\tikztotarget)
      },
      ncbar/.default=0.05cm,
  }

  \tikzset{square left brace/.style={ncbar=0.05cm}}
  
\usepackage{color}

\begin{document}
%
\title{Are Safety Guarantees in Neural Networks Safe? How to Compute Trustworthy Robustness Certifications}
%
\titlerunning{Trustworthy Robustness Certifications in NNs}

%
\author{Merkouris Papamichail\inst{1,2} \and
Konstantinos Varsos\inst{1,2} \and
Giorgos Flouris\inst{1} \and
Jo\~ao Marques-Silva\inst{3,4}}

\authorrunning{M. Papamichail et al.}

\institute{
  Foundation for Reasearch and Technology - Hellas, Heraklion, Greece
  \email{\{mercoyris,varsosk,fgeo\}@ics.forth.gr} \and
  University of Crete, Heraklion, Greece \and
  Catalan Institution for Research and Advanced Studies, Barcelona, Spain
  \email{jpms@icrea.cat}\and
  University of Lleida, Lleida, Spain
}

\maketitle              
    \setcounter{footnote}{0}    
    \begin{abstract}
      A primary challenge in AI safety is the existence of adversarial examples ---slightly distorted inputs that cause a neural network (NN) to misclassify. To mitigate this problem, recent research focuses on the computation of \emph{robustness certifications}, which, for a given input, determine the largest distortion the input may receive without breaking the network's prediction. Robustness certifications can be interpreted as an axis-aligned hyper-rectangle (\emph{multi-dimensional intervals}). Most existing approaches focus on maximizing the certification's \emph{volume}, but recent intractability results prohibit the computation of volume-optimal certifications in reasonable time. We introduce the \emph{apothem} measure 
---the minimum slack between the input and one of the interval's faces,
and show how to compute \emph{apothem-optimal} certifications in a \emph{linear} number of calls to a NN verifier (\emph{oracle}) w.r.t. the input domain's diameter. Moreover, we prove that we \emph{cannot} have a volume-optimal, oracle-based algorithm, even if we \emph{discard} the oracle costs. Also, we introduce \emph{dual certifications} ---an interval including all instances of a class--- thus providing \emph{apothem-minimum} upper bounds to a robustness certification. Further, we present the \textbf{\texttt{Parallelepipe\-doNN}}\footnote{\url{https://github.com/merkouris148/parallelepipedonn}} system, which we evaluate on the standard MNIST and Fashion MNIST benchmarks. A preliminary comparison with existing work on the same datasets reveals at least two-fold improvement w.r.t. the minimum edge length.


\keywords{AI Safety \and Adversarial Robustness \and Interval Algebra.}
    \end{abstract}

    \section{Introduction}
    \label{sec:introduction}
    Neural Networks (NNs) have been successfully applied to a variety of machine learning tasks, achieving both high performance and empirical accuracy. However, it has been observed that slight distortions of the input can break the network's prediction. Such distorted inputs are called  \emph{adversarial examples} \cite{szegedy_intriguing_2014,goodfellow_explaining_2015} and motivated research on \emph{robustness certification}. 
In robustness certification we are presented with an input and must compute the largest distortion the input may receive without causing a misclassification. By calculating tolerance intervals for each dimension and taking their Cartesian product, we are left with an \emph{axis-aligned hyper-cube}, or \emph{multi-dimensional interval}.

This geometrical interpretation allows us to express robustness certification as an \emph{optimization problem} such that finding the best certification corresponds to optimizing w.r.t. a given measure. The most widely used approach in the area is to express robustness certification as a constrained convex optimization problem, and then solving it numerically using first order differential methods \cite{wong_provable_2018,liu_certifying_2019,li_towards_2022,kabaha_maximal_2023}. This approach also leverages powerful NN frameworks (e.g., PyTorch, TensorFlow) for efficiently computing the necessary gradients. Despite its efficiency, this approach comes at a cost: in order to be applied to real world NNs, the network needs to be \emph{linearly relaxed}. As the ReLU activation function is non-differentiable, traditional NNs can only be described using boolean variables that lead to non-linear constraints, and thus linear relaxation is in effect an approximation. In particular, linearly relaxing a NN results in false positives, flagging safe inputs as adversarial examples. This leads to artificially smaller robustness certifications, underestimating the network's local stability.

This uncertainty hinders the applicability of robustness certifications in practice. For example, consider the common robustness certifications' quality measure, the \emph{volume} \cite{wong_provable_2018,liu_algorithms_2019,li_towards_2022}. Suppose that in order to consider the network robust, we need to achieve a volume threshold of $v_1$. However, we calculated a certification of volume $v_2$, with $v_2 < v_1$. \emph{Does the network meet the desired safety guarantees? Should we retrain it, or is its robustness underestimated?} Since volume optimality cannot be guaranteed, these questions cannot be answered.

This work aims to mitigate this exact problem. Firstly, we introduce the \emph{apothem} measure,
which measures the minimum ``slack'' between the input and one of the faces of the hyper-rectangle.
Next, we show how to compute apothem-optimal robustness certifications by iteratively querying a NN verifier (oracle). For arbitrary certifications, this can be done in a \emph{linear} number of oracle calls w.r.t. the input domain's diameter. If we constrain ourselves to \emph{uniform} certifications (allowing only uniform distortion), we achieve apothem optimality in \emph{logarithmic} oracle calls w.r.t. the input domain's diameter. Further, we introduce the notion of \emph{dual certifications}, which covers all the inputs of a specific class. Hence, a dual certification includes all possible robustness certifications of a class, constituting an upper bound to any increasing quality measure. Finally, we show that optimality cannot be achieved in polynomial time for the volume metric by an oracle-based algorithm, even if we discard the oracle costs.

Returning to our original example, our apothem-optimal algorithms can be used to compute a robustness certification. If the threshold is $\varpi_1$, but we computed $\varpi_2$, with $\varpi_2 < \varpi_1$, we \emph{know} that the network is to blame and should be retrained. Moreover, if we need to use another quality measure, like volume, we can compute the dual certification. If the dual's volume is $v_2$ and we need to have $v_1$, with $v_1 > v_2$, we \emph{know} that \emph{no} such certification exists, for this NN.

We implemented our algorithms in the \texttt{\textbf{ParallelepipedoNN}} system, which we evaluate on the standard MNIST and Fashion MNIST datasets on multiple measures. Notably, a preliminary comparison with existing software reveals a two-fold improvement w.r.t. the certification's minimum edge length, and at least an order of magnitude improvement w.r.t. the certification's diameter.


    \section{Multi-Dimensional Intervals}
    \label{sec:intervals}
    We begin with some foundational results from Sunaga's \emph{Interval Algebra} \cite{sunaga_theory_1958}. Firstly, we generalize the $\leq ~\subseteq \mathbb{R} \times \mathbb{R}$ relation to multi-dimensional spaces. For two vectors $\bell, \bu \in \mathbb{R}^d$ we write $\bell \leq \bu$ \emph{iff} $\ell_i \leq u_i$ for all $i \in [d]$, and $\bell < \bu$ \emph{iff} $\ell_i < u_i$ for all $i \in [d]$. 
Multi-dimensional intervals are defined as follows:

\begin{definition}[Multi-Dimensional Intervals]
    \label{def:intervals}
    Let $\bell, \bu \in \mathbb{R}^d$, with $\bell \leq \bu$. An \emph{interval} $[\bell, \bu] \subset \mathbb{R}^d$ is the set of points $\mathbf{x} \in \mathbb{R}^d$ such that $\bell \leq \mathbf{x} \leq \bu$. Finally, $\intervals{d}$ denotes the space of $d$--dimensional intervals, i.e., $\intervals{d} = \{S \subset \mathbb{R}^d \mid \exists \ \bell, \bu \in \mathbb{R}^d, \ \bell \leq \bu, \ S = [\bell, \bu]\}$.
\end{definition}

The multi-dimensional intervals of Def.~\ref{def:intervals} are used to formalize the concept of robustness certification. Let $I_1, \dots, I_d$ be the sequence of (one-dimensional) tolerance intervals of a robustness certification. Recall that we can perturb freely each feature $i$ within $I_i$, without breaking the prediction. Concretely, this robustness certification can be represented by the multi-dimensional interval $I = I_1 \times \cdots \times I_d$.

Assume $\mathbb{F} \subset \mathbb{R}^d$ be the \emph{input domain} of a particular neural network. For a given input $\mathbf{x} \in \mathbb{F}$, we can express the robustness certifications of $\mathbf{x}$ as intervals of the form $[\mathbf{x} - \mathbf{a}, \mathbf{x} + \mathbf{b}]$, where $\mathbf{a}, \mathbf{b} \geq \mathbf{0}$. Intuitively, the vectors $\mathbf{a}, \mathbf{b}$ correspond to the maximum subtractive and additive perturbations not breaking the network's prediction. Inversely, for any interval $I$ containing $\mathbf{x}$, there are vectors $\mathbf{a}, \mathbf{b} \geq \mathbf{0}$, s.t. $I = [\mathbf{x} - \mathbf{a}, \mathbf{x} + \mathbf{b}]$. For the domain $\mathbb{F}$ and an input $\mathbf{x}$, $\intervalsdxu$ denotes the space of all the intervals of $\mathbb{F}$ containing $\mathbf{x}$.

For any $\mathbf{x} \in \mathbb{R}^d$, we call the interval $[\mathbf{x}, \mathbf{x}] = \{\mathbf{x}\}$ \emph{trivial}. For a radius $\rho > 0$, we call intervals of the form $[\mathbf{x} - \rho\mathbf{1}, \mathbf{x} + \rho\mathbf{1}]$ \emph{uniform}. Uniform intervals geometrically represent axis-aligned \emph{hyper-cubes}. Moreover, for any vector $\mathbf{v} \geq \mathbf{0}$, we call intervals of the form $[\mathbf{x} - \mathbf{v}, \mathbf{x} + \mathbf{v}]$ \emph{symmetric}. Finally, \emph{arbitrary} intervals are characterized by two vectors $\mathbf{a}, \mathbf{b} \geq \mathbf{0}$, s.t. $[\mathbf{x} - \mathbf{a}, \mathbf{x} + \mathbf{b}]$. See Fig. \ref{fig:intervals} (left).

\begin{figure}
    \centering
    \begin{subfigure}[b]{0.47\textwidth}
        \centering
        \includegraphics[scale=0.23]{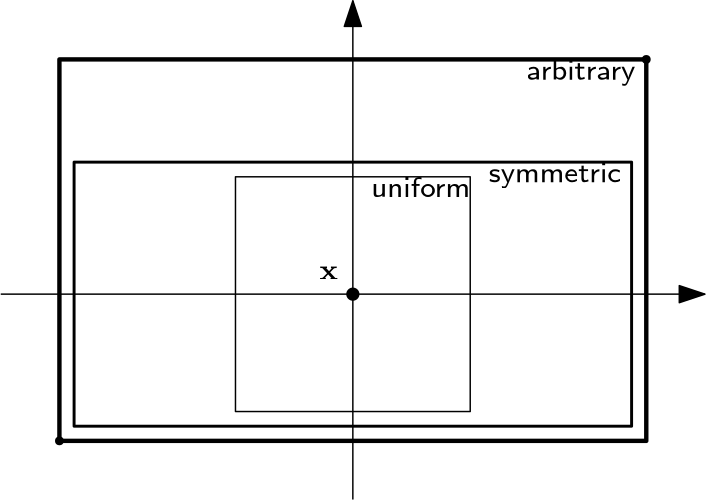}
    \end{subfigure}
    \hfill
    \begin{subfigure}[b]{0.47\textwidth}
        \centering
        \includegraphics[scale=0.23]{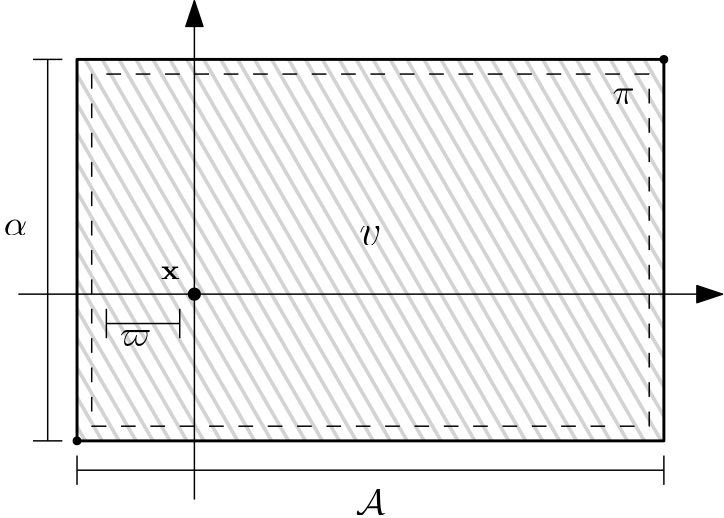}
    \end{subfigure}
    \caption{Left: Interval Families: uniform, symmetric, and arbitrary. Right: Interval Measures: $\mathcal{A}$ diameter, $\pi$ perimeter, $v$ volume, $\alpha$ minimum edge, $\varpi$ apothem.}
    \label{fig:intervals}
\end{figure}

\subsection{Operations on Intervals, and the Interval Lattice}

Recall that we denote the space of intervals in $\mathbb{F}$ that contain an input $\mathbf{x}$ with $\intervalsdxu$. The space $\intervalsdxu$ is highly structured by simple algebraic operations.

\begin{definition}
    \label{def:interval-ops}
    Let $[\bell, \bu], [\bm, \bn] \in \intervalsdxu$. We define the operations:
    \begin{enumerate}
        \item $[\bell, \bu] + [\bm, \bn] \overset{\Delta}{=} [\bell + \bm, \bu + \bn]$
        \item $[\bell, \bu] \sqcup [\bm, \bn] \overset{\Delta}{=} [\min\{\bell,\bm\}, \max\{\bu, \bn\}]$
        \item $[\bell, \bu] \sqcap [\bm, \bn] \overset{\Delta}{=} [\max\{\bell,\bm\}, \min\{\bu, \bn\}]$
    \end{enumerate}
\end{definition}

It is easy to see that $\intervalsdxu$ is \emph{closed} under the operations in Def. \ref{def:interval-ops}. The summation of two intervals, corresponds to the well-known \emph{Minkowski sum}. We sometimes abuse notation, writing $\mathbf{x} + [\bell, \bu]$, instead of $[\mathbf{x}, \mathbf{x}] + [\bell, \bu]$ (analogously for the other operations). The $\sqcup, \sqcap$ operations were introduced by Sunaga \cite{sunaga_theory_1958}, and reveal the underlying \emph{lattice} structure. A partially ordered set is a lattice when for any two elements the \emph{least upper bound (lub)} and \emph{greatest lower bound (glb)} exist. In the $\intervalsdxu$ latttice structure, the lub is given by the \emph{join} operation $\sqcup$, and the glb by the \emph{meet} operation $\sqcap$. 

\begin{theorem}[Interval Lattice \cite{sunaga_theory_1958}]
    \label{theo:interval-lattice}
    The interval space $\intervalsdxu$, under $\subseteq$ constitutes a complete \emph{lattice} with $\sqcup, \sqcap$ as the \emph{meet} and \emph{join} operations, respectively.
\end{theorem}

\subsection{Interval Objectives}
\label{subsec:interval-measures}

We proceed by discussing interval measures. A \emph{measure} $\mu\colon\intervalsdxu \to \mathbb{R}_{\geq 0}$ is a function assigning a positive real value to an interval. The trivial interval has to have zero measure, and a measure must be monotone w.r.t. set-inclusion. We also introduce the novel \emph{apothem} measure.

\begin{definition}[Interval Measures]
    \label{def:interval-objectives}
    Consider an interval $I = \mathbf{x} + [\bell, \bu] \in \intervalsdxu$, with $\bell \leq \bu$. Then we define the following measures:
    \begin{center}
        \begin{tabular}{p{11em}  l l}
            Apothem:& $\varpi(I)$& $= \min\Big\{\min_{i \in [d]} (x_i - \ell_i), \min_{j \in [d]} (u_j - x_j)\Big\}$\\
            Minimum Edge:&  $\alpha(I)$ &$= \min_{i \in [d]} u_i - \ell_i$\\
            Perimeter:& $\pi(I)$            &$ = \sum_{i \in [d]} u_i - \ell_i$\\
            Volume:&    $\vol(I)$           &$= \prod_{i \in [d]} u_i - \ell_i$.\\
            Diameter:&  $\mathcal{A}(I)$    &$= \max_{i \in [d]} u_i - \ell_i = \|u_i - \ell_i\|_\infty$
        \end{tabular}
    \end{center}
\end{definition}

All the measures of Def. \ref{def:interval-objectives} are depicted in Fig. \ref{fig:intervals} (right). In our experimental evaluation, we also use the \emph{average edge length} $\overline{E}(I)$, defined as $\overline{E}(I) = \pi(I)/d$. Apothem\footnote{
    In Euclidean geometry \emph{apothem} is a line between the polygon's center and a faces.
} measures the minimum ``slack'' between the input $\mathbf{x}$ and one of the faces of the corresponding hyper-rectangle. In the sequel, we see that a robustness certification's faces will be bounded by adversarial examples.
Hence, we can use the apothem to measure the minimum distance between the input $\mathbf{x}$ and the nearest adversarial example.
Prop. \ref{prop:numerical-geometric-mean} relates the above measures. We include omitted proofs in App. \ref{app:proofs}.

\begin{restatable}{proposition}{numericalgeometricmean}
    \label{prop:numerical-geometric-mean}
    For an interval $I \in \intervalsdxu$ and the measures of Def. \ref{def:interval-objectives} we have,
    \begin{equation}
        \label{eq:numerical-geometric-mean}
        \mathcal{A}(I) \geq \frac{1}{d}\cdot\pi(I) \geq \sqrt[d]{v(I)}\geq \alpha(I) \geq 2\cdot\varpi(I)
    \end{equation} 
\end{restatable}

    \section{Robustness Operators}
    \label{sec:robustness}
    Here we will work abstractly, in order to consider the combinatorial aspects of robustness certification and its relation to multi-dimensional intervals. In particular, consider the domain of a neural network $\mathbb{F}$ and a set of adversarial examples $\mathcal{V} \subset \mathbb{F}$. Moreover, let $\mathbf{x} \in \mathbb{F}$ be the given input. A robustness certification will be an interval $I \in \intervalsdxu$, s.t. $\mathbf{x} \in I$ \emph{and} $I \cap \mathcal{V} = \varnothing$.

The technicalities on how we obtain the set of adversarial examples $\mathcal{V}$ will be discussed in Section~\ref{sec:algorithms}. For now, let's assume that this set is complete, meaning that if we manage to exclude its members, there is no other input in the domain $\mathbb{F}$ breaking the prediction. This lets us focus on a subtle, but important problem. \emph{From all the candidate intervals $J \subseteq \mathbb{F} \setminus \mathcal{V}$, which one should we choose?}

We proceed in two steps. Firstly, we consider the case of excluding a single adversarial example $\mathbf{v} \in \mathcal V$, defining the apothem-optimal, ``small-step''\footnote{
    ``Small-step'' operators deal with one example, while ``big-step'' with sets.
}, \emph{constrain} operator. Then, we generalize for multiple adversarial examples, where we define the ``big-step'', \emph{bottom} operator, also ensuring apothem-optimality. Finally, we discuss dual certifications, returned by the big-step, \emph{top} operator.

\subsection{Small-Step Robustness}

\begin{figure}
    \centering
    \begin{subfigure}[b]{0.47\textwidth}
        \centering
        \includegraphics[scale=0.23]{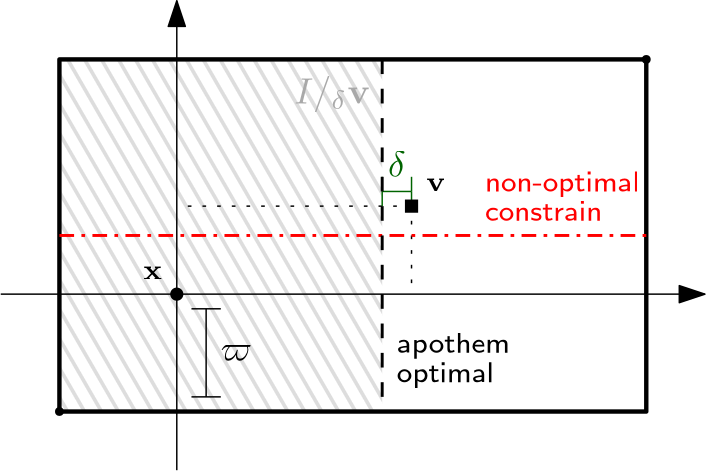}
    \end{subfigure}
    \hfill
    \begin{subfigure}[b]{0.47\textwidth}
        \centering
        \includegraphics[scale=0.23]{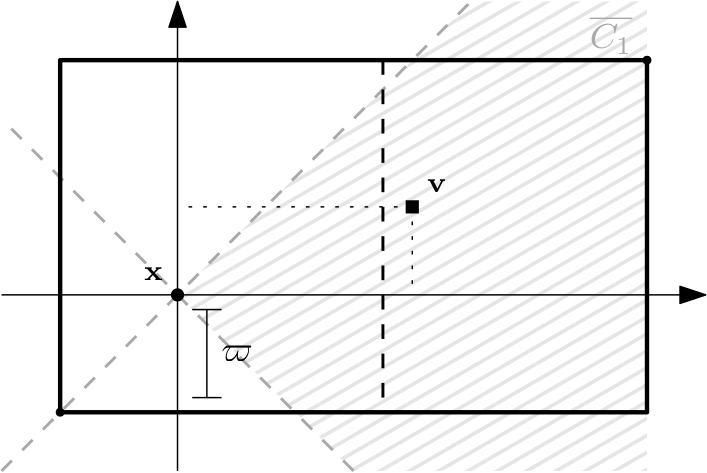}
    \end{subfigure}
    \caption{Left: The constrain operator. Right: How the domain space can be partitioned in cones. The cone $\overline{C_i}$ contains all the adversarial examples that result in modifying the $i$-th coordinate of the upper endpoint $\bu$.}
    \label{fig:constrain}
\end{figure}

Now we discuss our ``small-step'' robustness operator. For an input $\mathbf{x}$, consider an interval $I$ s.t. $\mathbf{x} \in I$. Moreover, let $\mathbf{v} \in I$ be an adversarial example we want to exclude. Our small-step robustness operator will do just that, excluding the adversarial example $\mathbf{v}$ from the interval $I$, in an \emph{apothem-optimal} manner. We call this operation \emph{constrain}, and denote it by $I/\mathbf{v}$. Subsequently, let $I^\prime = I/\mathbf{v}$.

Recall that the intervals are essentially sets of $2d$ inequalities, each corresponding to a lower or upper bound for one of the $d$ coordinates of the input space, written compactly as $I = \mathbf{x} + [\bell, \bu]$, with $\bell \leq \mathbf{0} \leq \bu$. To ensure that $\mathbf v\notin I'$, it suffices to ensure that one of these inequalities fails. To achieve apothem-optimality, we choose the dimension $k \in [d]$ for which $|v_k - x_k|$ is maximum. If $v_k - x_k >0$, then the upper bound must be reduced to exclude $\mathbf{v}$, by changing the $k$-th dimension of $\bu$ to produce a $\bu'$ that leaves $\mathbf v$ out of $I'$; otherwise the lower bound must be increased accordingly.

An important technicality here is that we use a precision constant $\delta>0$ to ensure a ``meaningful'' reduction of the interval. When modifying the upper bound, we set $u_k' = v_k - x_k - \delta$, and symmetrically for the lower bound. This serves two purposes. First, it excludes $\mathbf{v}$ in a ``strict sense'' (otherwise $\mathbf{v}$ would be in the border of $I'$). Second, it avoids infinitesimal modifications. As we will see later, this is necessary to guarantee the termination of our algorithms. This process is visualized in Fig.~\ref{fig:constrain} (left) and defined formally below.

\begin{definition}[Constrain Operator]
    \label{def:interval-exclusion}
    Consider an input $\mathbf{x} \in \mathbb{F}$, an interval $I = \mathbf{x} + [\bell, \bu]$, with $\bell \leq \mathbf{0} \leq \bu$, an adversarial example $\mathbf{v} \in I$, and a precision constant $\delta > 0$. 
    Let $k= argmax_{i \in [d]}\{ |v_i - x_i| \}$.
    We define the \emph{constrain operation} $I/_\delta ~\mathbf{v} = \mathbf x + [\bell',\bu']$, where:
    \begin{itemize}
        \item If $v_k - x_k > 0$ then $\bell' = \bell$, $u_i' = u_i$ for all $i \in [d] \setminus \{k\}$ and $u_k' = \max\{0, v_k - x_k - \delta\}$.
        \item If $v_k - x_k < 0$ then $\ell_i' = \ell_i$ for all $i \in [d] \setminus \{k\}$, $\ell_k' = \min\{0,v_k - x_k + \delta\}$, and $\bu' = \bu$. 
    \end{itemize}
\end{definition}

Note that $v_k - x_k \neq 0$ because, by construction, $\mathbf x \neq \mathbf v$. The reason for $\max$ and $\min$ in the definition is to ensure that $\mathbf{x} \in I^\prime$, i.e., $\bell' \leq \mathbf 0 \leq \bu'$. Subsequently, we omit $\delta$, writing $I/\mathbf{v}$. The operator's apothem-optimality is established below.

\begin{restatable}{proposition}{intervalexclusion}
    \label{prop:interval-exclusion}
    Consider an interval $I \in \intervalsdxu$ and a point $\mathbf{v} \in I$. For any interval $J \in \intervalsdxu$, with $\mathbf{v} \notin J$ and $J \subseteq I$, we have $\varpi(J) \leq \varpi(I/\mathbf{v})$.
\end{restatable}
\begin{proof}
    W.l.o.g. let $\mathbf{x} = \mathbf{0}$. For any $J \in \intervalsdxu$ and $\mathbf{v} \in \mathbb{F}$, it holds $\mathbf{v} \notin J$, \emph{iff} $\max_{i\in[d]} \abs{v_i} >\varpi(J)$. Also, $\varpi(I/\mathbf{v}) = \min\{\varpi(I), \max_{i\in[d]} \abs{v_i}\}$. \hfill$\blacksquare$
\end{proof}

\subsection{Big-Step Robustness}

\begin{figure}
    \centering
    \begin{subfigure}[b]{0.47\textwidth}
        \centering
        \includegraphics[scale=0.23]{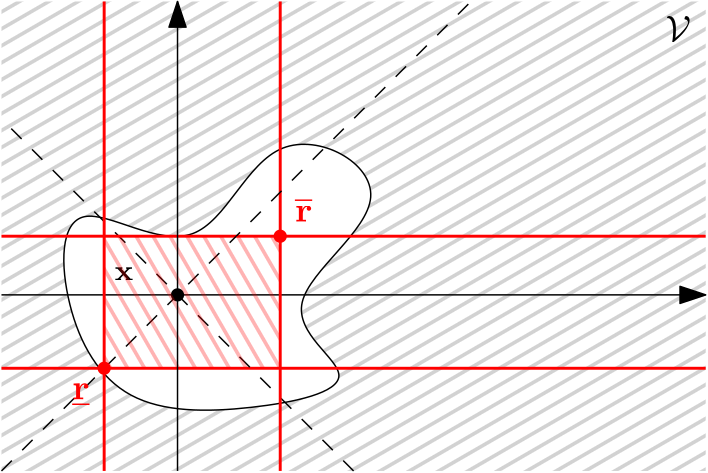}
    \end{subfigure}
    \hfill
    \begin{subfigure}[b]{0.47\textwidth}
        \centering
        \includegraphics[scale=0.23]{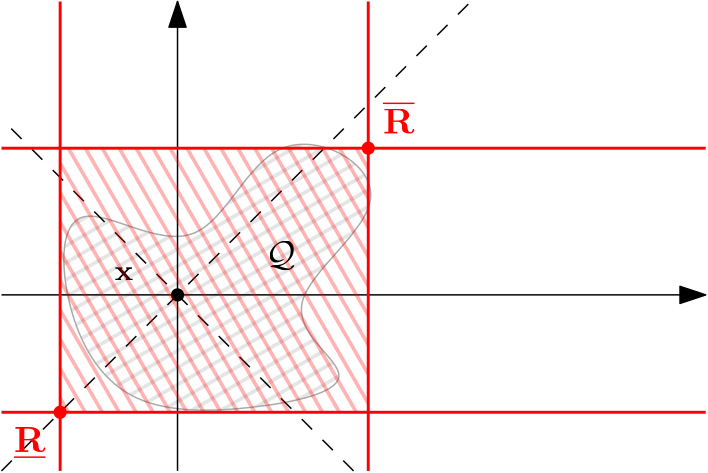}
    \end{subfigure}
    \caption{Left: The bottom interval $\mathcal{V}\bot\mathbf{x}$. Right: The dual top interval $\mathcal{Q}\top\mathbf{x}$.}
    \label{fig:big-step}
\end{figure}

Now we revert to the original problem of handling a set of adversarial examples. Observe that the small-step constrain operator implicitly defines a partition of $\mathbb{F}$ into $2d$ regions\footnote{%
    \label{foot:geometry}%
    Corresponding to polyhedral cones, $C = \{\mathbf{x} \in \mathbb{R}^n \mid A\mathbf{x} \leq \mathbf{0}\}$, $A \in \mathbb{R}^{m \times n}$.
}, denoted by $\overline{C_i}, \underline{C_i}$, for $i \in [d]$, see Fig.~\ref{fig:constrain} (right).

\begin{definition}[Cone Partition]
    \label{def:cones}
    Consider the domain $\mathbb{F} \subset \mathbb{R}^d$, and an input $\mathbf{x} \in \mathbb{F}$. The cone partition w.r.t. $\mathbf{x}$ is defined by the sequence $\{\underline{C_i}, \overline{C_i}\}_{i \in [d]}$, where
    \begin{equation}
        \label{eq:cones}
        \begin{split}
            & \underline{C}_i = \left\{\mathbf{v} \in \mathbb{F}~\big\vert~ v_i \leq x_i, \forall j \in [d] ~\abs{v_i - x_i} \geq \abs{v_j - x_j} \right\}, \\
            & \overline{C}_i = \left\{\mathbf{v} \in \mathbb{F}~\big\vert~ v_i \geq x_i, \forall j \in [d] ~\abs{v_i - x_i} \geq \abs{v_j - x_j} \right\}.
        \end{split}
\end{equation}
\end{definition}

For a given set of adversarial examples $\mathcal{V}$, let $\mathcal{V}\bot\mathbf{x}$ denote the \emph{bottom interval}, where $(\mathcal{V}\bot\mathbf{x}) \cap \mathcal{V} = \varnothing$, see Figure \ref{fig:big-step} (Left). Since $\mathbf{x} \in (\mathcal{V}\bot\mathbf{x})$, we write the bottom interval as $\mathcal{V}\bot\mathbf{x} = \mathbf{x} + [-\underline{\mathbf{r}}, \overline{\mathbf{r}}]$, with $\underline{\mathbf{r}}, \overline{\mathbf{r}} \geq \mathbf{0}$. The $k$-th coordinate of the upper endpoint $\overline{\mathbf{r}}$ is set to the \emph{minimum} coordinate-wise absolute distance $\abs{x_k - v_k}$. We work similarly for the $k$-th coordinate of the lower endpoint $\underline{\mathbf{r}}$.

\begin{definition}[$\bot$-Operator]
    \label{def:bot-operator}
    Consider a domain $\mathbb{F}$ and an input $\mathbf{x} \in \mathbb{F}$. Let $\mathcal{V} \subset \mathbb{F}$ be a set of adversarial examples. Moreover, assume the values $\underline{r_i},\overline{r_i} > 0$, where $\uradi = \inf\{x_i - v_i \mid \mathbf{v} \in \underline{C_i} \cap \mathcal{V}\}$ and $\oradi = \inf\{v_i - x_i \mid \mathbf{v} \in \overline{C_i} \cap \mathcal{V}\}$. Then, $\mathcal{V}\bot\mathbf{x}$ denotes the bottom interval, where $\mathcal{V}\bot\mathbf{x} = \mathbf{x} + [-\underline{\mathbf{r}}, \overline{\mathbf{r}}]$.
\end{definition}

Observe that, by definition, the bottom interval can be \emph{unbounded}. Indeed, if one of the cones is empty, e.g., $\overline{C_i} = \varnothing$, then the corresponding coordinate-wise absolute distance approaches infinity, i.e., $\overline{r_i} \to +\infty$. This happens trivially, when inputs lay at the border of the domain $\mathbb{F}$, thus $\overline{C_i} \cap \mathbb{F} = \varnothing$. Hence, no adversarial examples exist. Note that unbounded intervals can still have bounded apothem, unless equal to $\mathbb{R}^d$. We discuss this issue again in Sec. \ref{sec:implementation}. Subsequently, we show that the bottom operator is apothem optimal.

\begin{theorem}[$\bot$-Operator Optimality]
    \label{theo:bot-operator}
    Consider a domain $\mathbb{F}$ and an input $\mathbf{x} \in \mathbb{F}$. Let $\mathcal{V} \subset \mathbb{F}$ be a set of adversarial examples. The following hold:
    \begin{enumerate}
        \item $(\mathcal{V}\bot\mathbf{x}) \cap \mathcal{V} = \varnothing$.
        \item For every $J \in \intervalsdxu$, s.t. $\mathbf{x} \in J$ and $J \cap \mathcal{V} = \varnothing$, we have $\varpi(J) \leq \varpi(\mathcal{V}\bot\mathbf{x})$.
    \end{enumerate}
\end{theorem}
\begin{proof}
    The first desideratum is trivial from Def. \ref{def:bot-operator}. We prove the second. W.l.o.g. let the apothem be achieved by the coordinate $\overline{r_k}$. From the Def. \ref{def:bot-operator} there is a adversarial example $\mathbf{v} \in \mathcal{V}$, s.t. $\overline{r_k} = v_k$. Assume the interval $J = \mathbf{x} + [\bell, \bu]$, $\bell \leq \mathbf{0} \leq \bu$, with $J \cap \mathcal{V} = \varnothing$. For sake of contradiction, $\varpi(J) > \varpi(\mathcal{V}\bot\mathbf{x})$. It holds, $u_k \geq \varpi(J) > r_k$. Thus, $u_k > v_k$. Since $J \cap \mathcal{V} = \varnothing$, there is a coordinate $j$, s.t. $u_j \leq v_j$. However, $\varpi(J) \leq u_j \leq v_j \leq v_k \leq r_k = \varpi(\mathcal{V}\bot\mathbf{x}) < \varpi(J)$, where we have $v_j \leq v_k$, since $\mathbf{v} \in \overline{C}_k$. A contradiction.
    \hfill $\blacksquare$
\end{proof}

As we discussed before, we can \emph{translate} the bottom operator to consecutive constrain operations. Below we establish that the resulting interval will always include the bottom interval, regardless of the order with which we perform the constrains. This enables us to apply the bottom operator in practice. Further, since in real-world cases adversarial examples will arrive in an on-line manner, we can exclude them upon arrival using the constrain operator.

\begin{restatable}{proposition}{bottranslation}
    \label{prop:bot-translation}
    Consider a domain $\mathbb{F}$, and an input $\mathbf{x} \in \mathbb{F}$. Let $\mathcal{V} \subset \mathbb{F}$ be a set of adversarial examples. We assume an ordering in $\mathcal{V}$, i.e., $\mathcal{V} = \mathbf{v}_1, \mathbf{v}_2, \dots, \mathbf{v}_n$. Let $I = \mathbb{F}~ /\mathbf{v}_1~ /\mathbf{v}_2~ \cdots ~/ \mathbf{v}_n$\footnote{
        We assume left associativity of the $/$ operator.
    }. Then, $\mathbf{x} \in I$, $I \cap \mathcal{V} = \varnothing$, and $I \supseteq \mathcal{V}\bot\mathbf{x}$.
\end{restatable}

\subsection{Dual Certifications}

Here we discuss the dual problem, which enables us to efficiently compute upper bounds to the robustness certification. We consider the inverse situation: for the domain $\mathbb{F}$ and input $\mathbf{x}$, we take a set of positive inputs $\mathcal{Q} \subset \mathbb{F}$, and compute an interval $I \in \intervalsdxu$, s.t. $\mathbf{x} \in I$ and $I \cup \mathcal{Q} = I$.

The interval $I$ is obtained by the top operator $\mathcal{Q}\top\mathbf{x}$, see Fig. \ref{fig:big-step} (Right). We work symmetrically to the bottom operator. In particular, the dual certification can be written in the form $\mathcal{Q}\top\mathbf{x} = \mathbf{x} + [-\underline{\mathbf{R}}, \overline{\mathbf{R}}]$, with $\underline{\mathbf{R}}, \overline{\mathbf{R}} \geq \mathbf{0}$. Now the $i$-th coordinate $\overline{R_i}$ is set to the \emph{maximum} coordinate-wise absolute distance between $\mathbf{x}$ and the positive points of $\mathcal{Q}$. For $\underline{\mathbf{R}}$ we work similarly.

\begin{definition}[$\top$-Operator]
    \label{def:top-operator}
    Consider a domain $\mathbb{F}$ and an input $\mathbf{x} \in \mathbb{F}$. Let $\mathcal{Q} \subset \mathbb{F}$ be a set of positive points. Moreover, assume the values $\underline{R_i}, \overline{R_i} > 0$, where $\underline{R_i} = \sup\{x_i - q_i \mid \mathbf{q} \in \mathcal{Q}\}$ and $\overline{R_i} = \sup\{q_i - x_i \mid \mathbf{q} \in \cap \mathcal{Q}\}$. Then, $\mathcal{Q}\top\mathbf{x}$ denotes the top interval, where $\mathcal{Q}\top\mathbf{x} = \mathbf{x} + [-\underline{\mathbf{R}}, \overline{\mathbf{R}}]$.
\end{definition}

Note that since $\mathcal{Q} \subset \mathbb{F}$, the top interval is \emph{always bounded}. We prove the optimality of the top operator in Thm. \ref{theo:top-operator}. However, we follow a different approach than the Thm. \ref{theo:bot-operator}. First, we prove the correctness of the top operator. Then, we show that the bottom operator returns a \emph{minimal} interval that includes all the points in $\mathcal{Q}$. Finally, we show that any minimal interval including $\mathcal{Q}$ is \emph{unique}. This results in apothem optimality.

\begin{restatable}{theorem}{topoperator}\emph{\textbf{($\top$-Operator Optimality})}
    \label{theo:top-operator}
    Consider a domain $\mathbb{F}$ and an input $\mathbf{x} \in \mathbb{F}$. Let $\mathcal{Q} \subset \mathbb{F}$ be a set of positive points. The following hold:
    \begin{enumerate}
        \item $(\mathcal{Q}\top\mathbf{x}) \cup \mathcal{Q} = (\mathcal{Q}\top\mathbf{x})$.
        \item There is no interval $J \in \intervalsdxu$, s.t. $\mathbf{x} \in J$, $J \cup \mathcal{Q} = J$, and $J \subset (\mathcal{Q}\top\mathbf{x})$.
        \item Let $I, J \in \intervalsdxu$ two intervals, s.t. $(\mathcal{Q} \cup \{\mathbf{x}\}) \subseteq I, J$. Moreover, let $I, J$ be \emph{minimal}, w.r.t. set inclusion, achieving this property. Then, $I = J$.
        \item For every $J \in \intervalsdxu$, s.t. $\mathbf{x} \in J$, and $J \cup \mathcal{Q} = J$, then $\varpi(J) \geq \varpi(\mathcal{Q}\top\mathbf{x})$.
    \end{enumerate}
\end{restatable}
\begin{proof}[Sketch]
    \emph{1.} Immediate consequence of Def. \ref{def:top-operator}. \emph{2.} We prove that if $J \subset (\mathcal{Q}\top\mathbf{x})$, then there is some $\mathbf{q} \in Q$, s.t. $\mathbf{q} \notin J$. \emph{3.} Let $I, J$ be minimal intervals. Then we prove that the intersection $I \cap J$ also includes $\mathcal{Q} \cup \{\mathbf{x}\}$. Thus, reaching a contradiction. \emph{4.} Follows from the previous results.
    \hfill $\blacksquare$
\end{proof}

Again, the top interval can be translated in a sequence of small-step interval operations. In particular, we use the join operator $\sqcup$ of Def. \ref{def:interval-ops} to include the points in $\mathcal{Q}$. We abuse the notation writing $I \sqcup \mathbf{q}$, instead of $I \sqcup [\mathbf{q}, \mathbf{q}]$.

\begin{restatable}{proposition}{toptranslation}
    \label{prop:top-translation}
    Consider a domain $\mathbb{F}$ and an input $\mathbf{x} \in \mathbb{F}$. Let $\mathcal{Q} \subset \mathbb{F}$ be a set of positive points. We assume an ordering in $\mathcal{Q}$, i.e., $\mathcal{Q} = \mathbf{q}_1, \mathbf{q}_2, \dots, \mathbf{q}_n$. Let $I = [\mathbf{x}, \mathbf{x}] \sqcup \mathbf{q}_1 \sqcup \mathbf{q}_2 \cdots \sqcup \mathbf{q}_n$\footnote{
        We assume left associativity of the $\sqcup$ operator.
    }. Then, $I = \mathcal{Q}\top\mathbf{x}$.
\end{restatable}

We close our discussion on the dual certification by establishing the (weak) duality of the bottom and top operators in the following theorem.

\begin{restatable}{theorem}{duality}\emph{\textbf{(Weak $\bot,\top$-Duality)}}
    \label{theo:duality}
    Consider an input domain $\mathbb{F}$, and a partition into positive points $\mathcal{Q}$ and adversarial examples $\mathcal{V}$, i.e., $\mathcal{Q} \cap \mathcal{V} = \varnothing$ and $\mathcal{Q} \cup \mathcal{V} = \mathbb{F}$. Then, for an input $\mathbf{x} \in \mathbb{F}$, we have $(\mathcal{V}\bot\mathbf{x}) \cap \mathbb{F} \subseteq \mathcal{Q}\top\mathbf{x}$ and for every increasing measure $\mu\colon\intervalsdxu\to\mathbb{R}_{\geq 0}$ of Def. \ref{def:interval-ops}, we have $\mu[(\mathcal{V}\bot\mathbf{x})  \cap \mathbb{F}] \leq \mu[\mathcal{Q}\top\mathbf{x}]$.
\end{restatable}

Note that since the bottom interval may be unbounded, we take the intersection $(\mathcal{V}\bot\mathbf{x}) \cap \mathbb{F}$. In the sequel, we discuss how the above observations can be applied to adversarial robustness of NN classifiers. There, all the points of a particular class $c$ are positive. The rest of the points are adversarial examples.


    \section{Algorithms}
    \label{sec:algorithms}
    In this section we discuss how the notions of Sec. \ref{sec:robustness} can be applied to compute the adversarial robustness of neural network classifiers. This is achieved by encoding the bottom and top operators as small-step interval operations, applying Prop. \ref{prop:bot-translation} and \ref{prop:top-translation}. However, for NNs we cannot have explicit descriptions of the adversarial examples $\mathcal{V}$ or positive points $\mathcal{Q}$. Instead, we only have the description of the NN, its weights, biases and activation functions. Using this, we construct first order formulas $\mathscr{Q}, \mathscr{S}$ that are satisfied only for points in $\mathcal{V}, \mathcal{Q}$, respectively.


ReLU NNs can be expressed as a set of inequalities with real and boolean variables. This formalism is called in the literature \emph{mixed integer linear programming (MILP)} \cite{lomuscio_approach_2017,fischetti_deep_2018}. Using existing MILP solvers (e.g. GLPK\footnote{
    \url{https://www.gnu.org/software/glpk/}
}), MILPs can be solved efficiently and accurately. Moreover, specific tools have been developed and optimized for NN MILPs. One such tool is the Marabou \cite{katz_marabou_2019,wu_marabou_2024} software for NN verification, which we also use in our \texttt{\textbf{ParallelepipedoNN}} software.

Therefore, despite not having an explicit description of the $\mathcal{V}, \mathcal{Q}$ sets, we can construct \emph{oracles} (first-order formulas, solvable with a MILP solver) $\mathscr{V}, \mathscr{Q}$, which either \emph{assert} that an interval $I$ is included in a set $\mathcal{P} \in \{\mathcal{V}, \mathcal{Q}\}$, or return a \emph{counterexample} $\mathbf{p} \in I \setminus\mathcal{P}$. Then, we use the counterexample $\mathbf{p}$ to refine the current interval using an operator $\square \in \{/, \sqcup\}$. 

\subsection{From Abstract Sets to Neural Networks}

Assume the input dimension $\din$ and the output dimension $\dout$. Let $\mathbb{F} \subset \mathbb{R}^{\din}$ denote the input (feature) space and $\mathbb{S} \subset \mathbb{R}^{\dout}$ denote the output (score) space. A NN is a function of the form $\sigma\colon\mathbb{F} \to \mathbb{S}$. In this work we are interested in ReLU activated NNs. For the one-dimensional case, the ReLU function $r\colon\mathbb{R} \to \mathbb{R}_{\geq 0}$ is defined as $r(x) = \max(x, 0)$. For the $d$-dimensional case we have $\mathbf{r}\colon\mathbb{R}^d \to \mathbb{R}^d_{\geq \mathbf{0}}$, where $\mathbf{r}(\mathbf{x}) = (r(x_1), \dots, r(x_d))$. We follow the definition in \cite{fefferman_reconstructing_1994}.

\begin{definition}[Neural Network Classifiers]
    \label{def:mlp}
    A \emph{neural network} $\sigma \colon \mathbb{F} \to \mathbb{S}$, with $\mathbb{F} \subset \mathbb{R}^{\din}, \mathbb{S} \subset \mathbb{R}^{\dout}$ is described as the tuple $\sigma = \angles{L, D, W, B}$, where:
    \begin{itemize}
        \item $L \in \mathbb{N}$ denotes the number of \emph{layers}.
        \item $D$ denotes the \emph{architecture}, i.e. a sequence $\din = d_0, d_1, \dots, d_{L-1}, d_L = \dout$.
        \item $W$ denotes the sequence of \emph{weight matrices}, s.t. $W^{(i)} \in \mathbb{R}^{d_i \times d_{i-1}}$, for $i \in [L]$
        \item $B$ denotes a sequence of \emph{biases} s.t.\ $\mathbf{b}^{(i)} \in \mathbb{R}^{d_i}$ for $i \in [L]$.
    \end{itemize}
    The value of $\sigma(\mathbf{x})$ is given as the value $\sigma^{(L)}$ in the system of equations below.
    \begin{equation}
        \label{eq:mlp}
        \left.
        \begin{array}{ll}
             \sigma^{(0)} &= \mathbf{x} \\
             \sigma^{(i)} &= \mathbf{r}[W^{(i)}\sigma^{(i-1)} + \mathbf{b}^{(i)}],\quad \forall i \in [L]
        \end{array}
        \right\}
    \end{equation}
    For classification problems, we also assume a set of classes $\mathcal{C}$, s.t. $\abs{\mathcal{C}} = \dout$. A classifier is a function of the form $\kappa\colon\mathbb{F}\to\mathcal{C}$, where $\kappa(\mathbf{x}) = \arg\max_{i \in \dout} \sigma(\mathbf{x})$.
\end{definition}

As discussed we can formulate a NN as a MILP to formally describe the input/output relation of the neural network as a logic formula. Let $\nu$ be that formula. For every input $\mathbf{x} \in \mathbb{F}$ and every output $\mathbf{y} \in \mathbb{S}$ the pair \emph{satisfies} $\nu$ if and only if the network $\sigma(\cdot)$ with input $\mathbf{x}$ outputs $\mathbf{y}$. Now assume a fixed input $\mathbf{x} \in \mathbb{F}$ and an interval $I = [\bell,\bu] \in \intervals{\din}\vert^{\mathbb{F}}_{\mathbf{x}}$, s.t. $\mathbf{x} \in I$. We have the following.
\begin{align}
    \mathscr{V}_{I, \nu}(\mathbf{p}) &\equiv 
        ~\left( \bigwedge_{i \in [\din]} \ell_i < p_i \land p_i < u_i \right)~ \land
        ~\nu(\mathbf{p}, \mathbf{y})~ \land
        ~\left( \bigvee_{j \neq c} y_j - y_{c} > \epsilon \right)
        \label{eq:sound-oracle}\\
    \mathscr{Q}_{I, \nu}(\mathbf{p}) &\equiv 
        \left( \bigvee_{i \in [\din]} p_i < \ell_i  \lor u_i < p_i \right)~ \land ~\nu(\mathbf{p}, \mathbf{y})~ \land ~\left( \bigwedge_{j \neq c} y_{c} - y_j > \epsilon \right) \label{eq:complete-oracle}
\end{align}
Above, $\epsilon > 0$ is a precision constant. The predicate $\mathscr{V}_{I, \nu}$ is satisfied only by adversarial examples, i.e., for inputs that belong to the current interval $I$ and maximize a different score than the target class $c$. Symmetrically, the predicate $\mathscr{Q}_{I, \nu}$ is satisfied by a $c$-instance not included in the current interval $I$. When there is no solution in $I$ for $\mathscr{V}_{I, \nu}$, we reached a robust certification. Symmetrically, when there is no solution outside of $I$ for $\mathscr{Q}_{I, \nu}$, we reached a dual certification.

\subsection{Refine \& Check Algorithm}

\begin{algorithm}[t]
    \DontPrintSemicolon
    \caption{\textsc{RefineCheck}$(I_o, \mathscr{P}_{I, \nu}, \square)$}
    \label{algo:generic-traversal}
    \KwInput{We assume the following arguments:\newline
    $\bullet$ $I_o \in \intervalsdxu$, an initial interval.\newline
    $\bullet$ $\mathscr{P}_{I, \nu}$, a property to be falsified.\newline
    $\bullet$ $\square\colon \intervalsdxu \times \mathbb{F} \to \intervalsdxu$, a refinement operator.}
    \KwOutput{$I$,  an interval s.t. $\models \lnot \mathscr{P}_{I, \nu}$}

    $I \gets I_o$ 
    
    \While{$\exists \mathbf{p}$, s.t. $\mathbf{p} \models \mathscr{P}_{I, \nu}$}{ \label{algo-step:generic-oracle}
        $I \gets I ~\square~\mathbf{p}$ \label{algo-step:generic-refine}
    }
    \Return{$I$}
\end{algorithm}

We now have all the components to develop our \emph{refine \& check} algorithm, whose pseudocode is given in Alg.~\ref{algo:generic-traversal}. Starting from an initial interval $I_o$, the algorithm repeatedly checks whether the current interval $I$ \emph{falsifies} some given property $\mathscr{P}_{I, \nu}$. The check is delegated to a NN verifier (line \ref{algo-step:generic-oracle}), which receives as input the MILP encoding of the network, the interval $I$, and the property $\mathscr{P}_{I, \nu}$. If the property is not falsified, the verifier returns a \emph{counterexample} $\mathbf{p}$ that satisfies the property. The interval is then \emph{refined} using $\mathbf{p}$ via some given operator $\square$, producing the interval $I ~\square~ \mathbf{p}$ (line \ref{algo-step:generic-refine}). The procedure continues until the verifier asserts that $\mathscr{P}_{I, \nu}$ is falsified, in which case the current interval is returned.

Different instantiations of Alg. \ref{algo:generic-traversal} can be obtained by appropriately choosing its parameters. Eq. \eqref{eq:top-down-params} shows how to compute the \emph{apothem-optimal robustness certification}. In this setting, the initial interval is the whole domain $\mathbb{F}$, the property to be falsified is given by eq. \eqref{eq:sound-oracle}, and the refinement operator is the constrain operator of Def. \ref{def:interval-exclusion}. The correctness of eq. \eqref{eq:top-down-params} is given by Prop. \ref{prop:bot-translation}, which shows that the bottom operator can be realized as a sequence of constrain operations. We call Alg. \ref{algo:generic-traversal} instantiated in this manner as \textsc{Top-Down Search} (TDS). Finally, since we begin from the domain $\mathbb{F}$, TDS computes the intersection $(\mathscr{V}\bot\mathbf{x}) \cap \mathbb{F}$.
\begin{align}
    &(\mathscr{V}\bot\mathbf{x}) \cap \mathbb{F} =\textsc{RefineCheck}(\mathbb{F},\mathscr{V}_{I, \nu}, /) &\textsc{Top-Down Search (TDS)} \label{eq:top-down-params}\\
    &\mathscr{Q}\top\mathbf{x} = \textsc{RefineCheck}([\mathbf{x}, \mathbf{x}],\mathscr{Q}_{I, \nu}, \sqcup) &\textsc{Bottom-Up Search (BUS)} \label{eq:bottom-up-params}
\end{align}

Eq. \eqref{eq:bottom-up-params} presents the implementation for \emph{dual certifications}. The construction proceeds symmetrically. The interval is initialized to the trivial $[\mathbf{x}, \mathbf{x}]$, the property to be falsified is given by eq. \eqref{eq:complete-oracle}, and the refinement operator is the join operator $\sqcup$ of Def. \ref{def:interval-ops}. The correctness of eq. \eqref{eq:bottom-up-params} follows from Prop. \ref{prop:top-translation}, which shows that the top operator can be expressed as a sequence of join operations. We refer to the parametrization in eq. \eqref{eq:bottom-up-params} as \textsc{Bottom-Up Search} (BUS).

Observe that the abstract sets $\mathcal{V}$ and $\mathcal{Q}$ introduced in Sec.~\ref{sec:robustness} have now been translated into the verification properties $\mathscr{V}$ and $\mathscr{Q}$.
This was possible only because we managed to translate the ``global'', big-step, bottom and top operations, to ``local'', small-step, operations. Based on the small-step operators, we are able to construct the \emph{on-line} refine \& check method of Alg. \ref{algo:generic-traversal}. In this setting the set $\mathcal{P}$, that is $\mathcal{P} \in \{\mathcal{V}, \mathcal{Q}\}$, is constructed incrementally, and contains all the counterexamples $\mathbf{p}$ returned by the NN verifier during the execution of Alg. \ref{algo:generic-traversal}.

\subsection{The Complexity of Apothem Optimality}

Our \emph{refine \& check} algorithm allows us to calculate the computational complexity of both operators in a single proof. To that end, let $\mathcal{P}$ be the set of all counterexamples returned by the NN verifier, and $n = \abs{\mathcal{P}}$. Moreover, let $t_{\mathscr{P}_{I, \nu}}$ be the time consumed by the NN verifier (in the worst case) to test the satisfiability of the property $\mathscr{P}_{I, \nu}$. Finally, let $t_{\square}$ be the execution time of the refinement operation $\square$. Then, in each iteration, Alg. \ref{algo:generic-traversal} processes one counterexample, makes a query to the NN verifier, and applies a refinement operator.

\begin{restatable}{theorem}{genericcomplexity}
    \label{theo:generic-complexity}
    Let $n$ be the number of counterexamples returned by the NN verifier. Let $t_{\mathscr{P}_{I, \nu}}$ be the worst-case time to check whether property $\mathscr{P}_{I, \nu}$ holds. Finally, let $t_{\square}$ be the worst-case time required by the operator $\square$. Then, the refine \& check method of Alg. \ref{algo:generic-traversal} terminates after $O[n\cdot(t_{\square} + t_{\mathscr{P}_{I, \nu}})]$ steps. If $\square \in \{/, \sqcup\}$, then $t_{\square} = O(d)$, and Alg. \ref{algo:generic-traversal} concludes in $O[n d + n t_{\mathscr{P}_{I, \nu}}]$ time.
\end{restatable}
Although the number $n$ of counterexamples is unknown, we can bound it from above using the precision constant $\delta$. The constrain operator only modifies one inequality of the given interval at each refinement step, and each inequality can be modified \emph{at most} $\mathcal{A}(\mathbb{F}) / \delta$ times, where $\mathcal{A}(\mathbb{F})$ is the domain's diameter. Since each interval is defined by $2d$ inequalities, the total number of counterexamples is bounded by $2d \cdot \mathcal{A}(\mathbb{F})/\delta$. The join operator $\sqcup$ is more intervening, modifying all $2d$ inequalities at once, resulting in $\mathcal{A}(\mathbb{F}) / \delta$ counterexamples.

\begin{corollary}
    \label{cor:generic-complexity}
    Let $t_{\mathscr{P}_{I, \nu}}$ be the worst case time for verifying property $\mathscr{P}_{I, \nu}$. TDS terminates after $O[(\mathcal{A}(\mathbb{F}) / \delta)(d^2 + d \cdot t_{\mathscr{P}_{I, \nu}})]$ time. BUS terminates after $O[(\mathcal{A}(\mathbb{F}) / \delta)(d + t_{\mathscr{P}_{I, \nu}})]$ time.
\end{corollary}

\subsection{The Intractability of Volume Optimality}

From Thm. \ref{theo:generic-complexity}  we can compute an \emph{apothem optimal} robustness certification in $\textsf{poly}(n, d, t_{\mathscr{P}_{I, \nu}})$ time. Namely, even if NN verification can be computed in constant time, i.e., $t_{\mathscr{P}_{I, \nu}} = O(1)$, our algorithm becomes polynomial. This does \emph{not} hold for other measures, like the volume, where we \emph{cannot} have polynomial time volume-optimality, even if NN verification costs are ignored.

To see that, recall the abstract robustness certification problem, introduced in Sec. \ref{sec:robustness}. There, we had a set $\mathcal{V} \subset \mathbb{F}$ of adversarial examples and an input point $\mathbf{x} \in \mathbb{F}$. Our goal was to compute an interval $I \in \intervalsdxu$, s.t. $\mathbf{x} \in I$ and $\mathcal{V} \cap I = \varnothing$. For a constant $\gamma > 0$, we want to decide if there is $I \in \intervalsdxu$ with the above properties and $v(I) \geq \gamma$. This problem is closely related to \emph{Maximum Empty Rectangle (MER)}~\cite{naamad_maximum_1984,chazelle_computing_1986,backer_mono-_2010,chan_faster_2023} and \emph{query-Maximum Empty Square (q-MES)}~\cite{gutierrez_finding_2012} problems. Next, we establish an intractability result based on the proof of MER's intractability by Backer and Keil in \cite{backer_mono-_2010}.

\begin{restatable}{theorem}{soundmaxhard}
    \label{theo:sound-max-hard}
    Consider a set of adversarial examples $\mathcal{V} \subset \mathbb{F}$, an input $\mathbf{x} \in \mathbb{F}$, and a constant $\gamma > 0$. Moreover, let $n = \abs{\mathcal{V}}$, let $d$ be the dimension s.t. $\mathbb{F} \subset \mathbb{R}^d$, and let $t_{\mathscr{P}_{I, \nu}}$  be the worst-case time check whether property $\mathscr{P}_{I, \nu}$ holds. Then, the existence of a robustness certification $I \in \intervalsdxu$, with $\mathbf{x} \in I$, $I \cap \mathcal{V} = \varnothing$ and $v(I) \geq \gamma$ \emph{cannot} be decided in $\textsf{poly}(n, d, t_{\mathscr{P}_{I, \nu}})$ time, unless $\textbf{P}=\textbf{NP}$.
\end{restatable}
\begin{proof}[Sketch]
    We use the argument of \cite{backer_mono-_2010} reducing from the independent set problem to maximum volume adversarial robustness. In \cite{backer_mono-_2010} each vertex of a given graph maps to a dimension, each edge maps to a point. The graph has an independent set of size $k$ iff there is a sufficiently large empty interval.

    Thus, given a graph $G$, we use the methodology in \cite{backer_mono-_2010} to construct a set of adversarial examples $\mathcal{V}$. However, we also need to define an input $\mathbf{x}$ that must be included to the interval. We set $\mathbf{x} = \mathbf{0}$. From Lemma 2 in \cite{backer_mono-_2010} the origin will always be included in a maximum volume interval. Thus, our choice of $\mathbf{x}$ does not effect the validity of the argument. This concludes the reduction.
$\hfill\blacksquare$
\end{proof}

\subsection{Notes on Uniform Certifications}
\label{sub-sec:uniform}
We close this section with some notes on uniform robustness certifications (originally used in \cite{wong_provable_2018}), i.e., certifications of the form $I = [\mathbf{x} - \rho\mathbf{1}, \mathbf{x} + \rho\mathbf{1}]$ for $\rho > 0$ and some  input $\mathbf{x} \in \mathbb{F}$. In our setting, we can use binary search to determine the maximum radius excluding the adversarial examples in $\mathcal{V}$, or the minimum radius including the positive examples $\mathcal{Q}$. The top and bottom operators of Sec. \ref{sec:robustness}, \emph{when restricted to uniform certifications}, are denoted by $\top_{\mathbb{B}}$ and $\bot_{\mathbb{B}}$ respectively, while the corresponding algorithms $\mathbb{B}$-BUS and $\mathbb{B}$-TDS. These operators can be computed in logarithmic $\log(\mathcal{A}(\mathbb{F})/\delta)$ oracle calls. We use uniform certifications as a \emph{baseline} in our following experimentation.

    \section{Implementation}
    \label{sec:implementation}
    In this section, we review the \texttt{\textbf{ParallelepipedoNN}} system, where we implement the algorithms and operations discussed earlier. In particular, we present two case studies based on the MNIST~\cite{lecun_mnist_2010} and the Fashion MNIST~\cite{xiao_fashion-mnist_2017} datasets. Specifically, we train two ReLU NNs and compute robustness and dual certifications for each dataset-network pair. For robustness certifications, we use the TDS algorithm of eq. \eqref{eq:top-down-params} and its \emph{uniform variant} $\mathbb{B}$-TDS (see Subsec. \ref{sub-sec:uniform}). For dual certifications, we use the BUS algorithm of eq. \eqref{eq:bottom-up-params} and its \emph{uniform variant} $\mathbb{B}$-BUS. We comment on the CPU time, oracle queries and edge lengths.

\paragraph{Experimental Setup.} We ran our experiments in parallel on an Ubuntu 18.04 machine, with Intel Xeon E5-2640 v4 CPU at 2.394GHz with 38 cores, with 128GB RAM. We utilized 35 cores. \texttt{\textbf{ParallelepipedoNN}} is written in Python v3.8.16. We used the Marabou v2.0~\cite{katz_marabou_2019,wu_marabou_2024} NN verifier. Our implementation takes as input a NN in open neural network exchange (ONNX) v1.16.0\footnote{
        See \url{https://github.com/onnx/onnx}.
} format. For linear algebra computations, we used the NumPy v1.23.5 library. For visualization, we used the Matplotlib v3.7.2 library. The NNs used were trained from scratch, using TensorFlow v2.12.0. We evaluated all the algorithms using the same parameters. We set the precision constant $\delta$ to $\delta =0.1$, and a \emph{timeout} to 1 hour. We use 2 datasets, namely MNIST \cite{lecun_mnist_2010} and Fashion MNIST \cite{xiao_fashion-mnist_2017}. Both datasets consist of 28$\times$28 grayscale images, belonging to 10 classes. The NN architecture is $D = \angles{784, 32, 10, 10}$, w.r.t. Def. \ref{def:mlp}. The last layer is fully connected, without activation function. This corresponds to 25,450 trainable parameters. For training, we used the Adam \cite{kingma_adam_2015} algorithm, Glorot \cite{glorot_understanding_2010} weight initialization, and the Categorical Crossentropy loss. The achieved test-set accuracy is $94\%$ and $82\%$ for the MNIST and Fashion MNIST, respectively. For each NN, we randomly chose 5 images of the 10 classes of the test set (a total of 50 images per NN).

\begin{figure}[t]
        \includegraphics[scale=0.4]{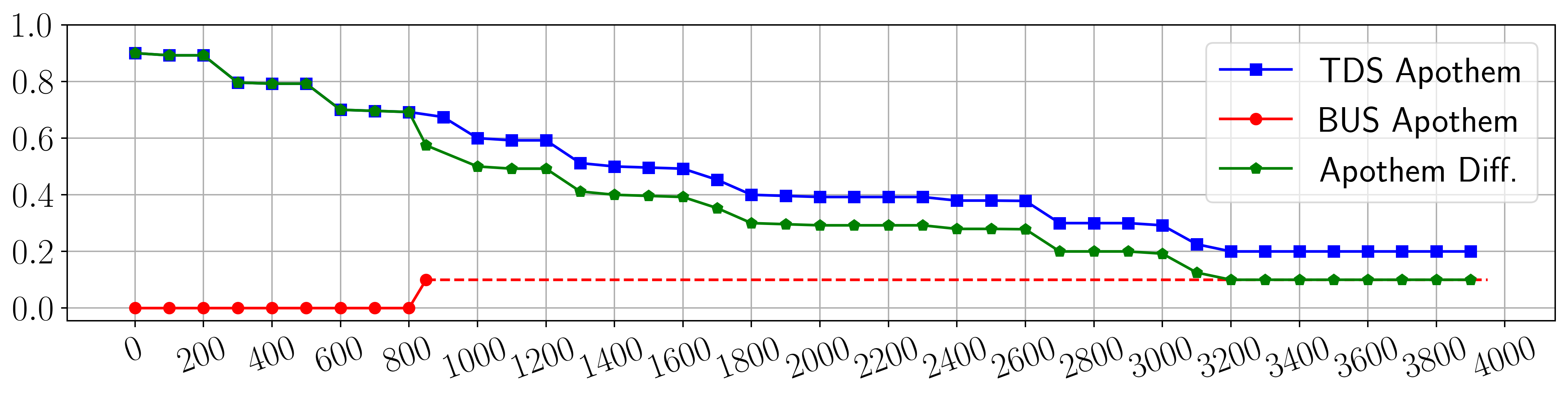}
    \caption{Apothems in TDS (blue), BUS (red), and their difference (green), w.r.t. iterations, on a ``7'' MNIST image. After 860, we extend BUS with dashed line.}
    \label{fig:evolution}
\end{figure}

\paragraph{Measures in Practice.}
Note that one totally black or totally white pixel suffices to place a grayscale image to the border of the input space, making $\mathscr{V}\bot\mathbf{x}$, and most measures, trivially unbounded. Therefore, for all measures, \emph{except the apothem}, we use the intersection $(\mathscr{V}\bot\mathbf{x}) \cap \mathbb{F}$. The apothem is computed in the (possibly unbounded) bottom interval $\mathscr{V}\bot\mathbf{x}$. This is achieved by excluding the interval's faces that stayed still during TDS' execution. Hence, \emph{only adversarially bounded faces are considered}, and the apothem returns the distance between the input and the \emph{nearest} adversarial example.

\paragraph{Running One Instance.}
In Fig. \ref{fig:evolution} we plot the \emph{apothem} of the TDS and BUS algorithms as a function of the \emph{iterations}, on a ``7'' MNIST image. TDS computes an apothem-maximum robustness certification, beginning from the whole domain, and gradually shrinking. In plateaus, the non-minimum edges are shrinking, thus the apothem stays constant. BUS computes an apothem-minimum dual certification.  Note that BUS terminates in much fewer iterations than TDS, because it updates all edges in each iteration. Finally, we depict the difference in apothems between the $i$-th robustness and dual certifications. In this case, the two apothems happen to converge at 0.1 distance.

\begin{table}
    \centering
    \begin{adjustbox}{width=\textwidth}
    \begin{tabular}{c l c c c c c c c c}
        \toprule
            \thead{Data}~~
            & \thead{Alg.}~~
            & \thead{Time\\\textbf{s}ec./\textbf{m}in.}~~
            & \thead{Verif.\\Perc.}~~
            & \thead{\# Verif.\\Calls}~~
            & \thead{Apoth.\\$\varpi$}~~
            & \thead{Min.\\Edge $\alpha$}~~
            & \thead{Avg.\\Edge $\overline{E}$}~~
            & \thead{Diam.\\$\mathcal{A}$}~~
            & \thead{Perim.\\$\pi$}~~
        \\ \midrule
            \multirow{4}{*}{\rotatebox{90}{\thead{MNIST}}} &
            \textsc{TDS} &
            56.93m &
            99\%&
            691.3 &
            \textbf{0.14} &
            \textbf{0.15} &
            \textbf{0.53} &
            \textbf{1.0} &
            \textbf{415.55}
        \\
            &
            $\mathbb{B}$--TDS &
            \textbf{20.16s} &
            98\% &
            \textbf{9.98} &
            0.04 &
            0.04 &
            0.05 &
            0.08 &
            35.76
        \\\cline{2-10}
            &
            \textsc{BUS} &
            16.13m &
            96\% &
            895.73 &
            1.0 &
            0.86 & 
            1.0 &
            1.0 &
            781.4
        \\
            &
            $\mathbb{B}$--BUS &
            \textbf{1.9s} &
            95\% &
            \textbf{4.0} &
            \textbf{0.81} &
            \textbf{0.81} &
            \textbf{0.83} &
            1.0 &
            \textbf{650.92}
        \\\midrule
            \multirow{4}{*}{\rotatebox{90}{\thead{Fashion\\MNIST}}} &
            \textsc{TDS} &
            9.19m &
            95\% &
            700.88 &
            0.01&
            0.07 &
            \textbf{0.66} &
            \textbf{1.0} &
            \textbf{516.85}
        \\
            &
            $\mathbb{B}$--TDS &
            \textbf{5.66s} &
            94\% &
            \textbf{9.62} &
            \textbf{0.1} &
            \textbf{0.1} &
            0.14 &
            0.19 &
            108.05
        \\\cline{2-10}
            &
            \textsc{BUS} &
            13.55m &
            96\% &
            757.36 &
            1.0 &
            0.9 &
            1.0 &
            1.0 &
            783.75
        \\
            &
            $\mathbb{B}$--BUS &
            \textbf{2.7s} &
            96\% &
            \textbf{4.0} &
            \textbf{0.81} &
            \textbf{0.81} &
            \textbf{0.89} &
            1.0 &
            \textbf{697.81}
        \\\bottomrule
    \end{tabular}
    \end{adjustbox}
    \caption{Experiments on MNIST and Fashion MNIST datasets. Averages over 50 samples from each datasets. Values are rounded up to two decimal places.}
    \label{tab:experiments}
\end{table}

\paragraph{Experimental Results.} In Tbl.~\ref{tab:experiments} we present results on the CPU time, the percentage of time spent in oracle calls, the number of oracle calls, the apothem $\varpi$, the minimum edge length $\alpha$, the average edge length $\overline{E} = \pi(I) / d$, the diameter $\mathcal{A}$, and the perimeter $\pi$. In Tbl.~\ref{tab:experiments}, for each problem we denote with bold the best value. For robustness certifications we observe a tradeoff between quality and performance. The uniform algorithm $\mathbb{B}$-TDS performs much faster than its arbitrary counterpart. This is expected by the discussion in Subsection \ref{sub-sec:uniform}. However, TDS achieves 3 times better quality, on average, as reported by $\alpha$ (in MNIST), the average edge length (in Fashion MNIST), etc. 

For the dual certification problem, we observe a different tendency. In both efficiency and quality the uniform variants perform better. We conjecture that this is due to numerical errors. The number of iterations and verification calls does not only affect the CPU time, but also introduces numerical errors. 

Observe that the majority of our computational time is consumed by oracle calls. Thus, our method inherits the oracle's shortcomings in performance and precision. Finally, in general, the robustness and dual certifications are not close. This is because the difference between the robustness and dual certifications show how close is a class' decision surface to an axis-aligned hyper-rectangle. Due to the complexity of NN decision surfaces, we expect this to happen rarely.

\paragraph{Comparison with Existing Software.}
A direct comparison with existing software is currently impossible. Software incompatibilities (e.g., between PyTorch, TensorFlow, and ONNX, or Marabou and ERAN, etc.) raise implementation obstacles. Moreover, most existing algorithms do not output their certifications, but only a final average. Thus, a direct experimental comparison falls outside this paper's scope. Nevertheless, a preliminary comparison reveals a \emph{two-fold} improvement w.r.t. the certification's minimum edge length, and \emph{an order of magnitude} improvement w.r.t. the certification's diameter.

We now make a preliminary comparison, \emph{as points of reference}. Liu et al. \cite{liu_certifying_2019} train 3 networks for MNIST, with 3 hidden layers and architectures $\angles{3 \times 100} = \angles{100, 100, 100}, \angles{3 \times 300}$ and $\angles{3 \times 500}$. In each case, they report minimum edge length of $\alpha \simeq 0.06$, for uniform certifications, and $\alpha \simeq 0.07$, for symmetric certifications. For Fashion MNIST, they use a network with architecture $\angles{3 \times 1024}$. They report $\alpha \simeq 0.026$, for uniform and $\alpha \simeq 0.028$, for symmetric certifications. The implementation, and the NN architectures of Li et al. \cite{li_towards_2022} are not publicly available. They train 3 NNs for MNIST (MNIST100, MNIST300, MNIST500), and one NN for Fashion MNIST (FMNIST100). They report $\alpha$ of $0.058$, $0.041$, $0.045$, and $0.072$, for MNIST100, MNIST300, MNIST500, and FMNIST100, respectively. Finally, Kabahala and Drachsler-Cohen \cite{kabaha_maximal_2023} report average edge length $\overline{E}$ of $0.196$ and $0.1$ for two MNIST NNs ($\angles{3 \times 50}$, $\angles{3 \times 100}$), whereas for two Fashion MNIST NNs ($\angles{3 \times 50}$, $\angles{3 \times 250}$) they report $\overline{E}$ of $0.191$ and $0.037$.

    \section{Related Work}
    \label{sec:related-work}
    Interval arithmetic \cite{sunaga_theory_1958,moore} has been used in NN adversarial analysis (e.g. \cite{wong_provable_2018,katz_reluplex_2022,xu_fast_2021}). However, they do not make explicit use of the underlying lattice structure \cite{sunaga_theory_1958}.

Adversarial robustness is the problem of \emph{deciding} whether an area is free of adversarial examples. Several approaches exist~\cite{meng_adversarial_2022}, such as abstraction (e.g. \cite{singh_abstract_2019}), satisfiability modulo theories (SMT) (e.g. \cite{katz_reluplex_2017,katz_reluplex_2022}), and MILP (e.g. \cite{katz_marabou_2019,wu_marabou_2024}).

In this work we focus on \emph{robustness certification} which is the optimization variant of the above. Most approaches rely on differentiable optimization, where the NN is firstly linearly relaxed (by taking the convex approximation of the ReLU activation) and then solved either by computing a feasible solution to the dual convex program \cite{wong_provable_2018}, or by the augmented Langrage method \cite{liu_certifying_2019,li_towards_2022}. These methods are scalable and can be applied to large NNs, but suffer from approximation errors \cite{salman_convex_2019} and lack optimality guarantees. More recent work~\cite{kabaha_maximal_2023} follows a hybrid approach, achieving \emph{maximality} by querying a NN verifier to check for adversarial examples before deciding on the edge to refine using differentiable methods. Optimization happens w.r.t. the average edge length $\overline{E}$, but optimality is not proved. Similar oracle-based iterative schemes have been applied in formal AI-explainability \cite{ignatiev_abduction-based_2019,wu_verix_2023,izza_distance-restricted_2024,izza_delivering_2024}, where an explanation can be interpreted as a special form of interval \cite{izza_delivering_2024}.

Existing work on robustness certification mostly focuses on optimizing the volume \cite{wong_provable_2018,liu_certifying_2019,li_towards_2022}, through optimizing minimum edge length. Deciding on the adversarial robustness of a \emph{uniform} area is \textbf{NP}-hard \cite{katz_reluplex_2017}, and even approximating the optimal $\alpha$ is intractable for uniform certifications \cite{weng_towards_2018}. However, volume optimality stays intractable even if the oracle costs are discarded.

    \section{Conclusions}
    \label{sec:conclusions}
    We proposed a method to compute trustworthy robustness certifications by ensuring optimality w.r.t. the apothem. Moreover, we introduced the notion of dual certifications, which gives an upper bound to the certification's quality. We also developed the \texttt{\textbf{ParallelepipedoNN}} system, which was evaluated over the MNIST and Fashion MNIST datasets. Our evaluation shows promising results w.r.t. existing implementations. Nevertheless, a direct experimental comparison is needed. We leave the latter as future work.

    \clearpage 
    %
    %
    %
    \bibliographystyle{splncs04}
    \bibliography{bibliography-v1-8-4-2026}

    \appendix
    
    \section{Omitted proofs}
    \label{app:proofs}
    \subsection{Proofs of Section \ref{sec:intervals}}

\numericalgeometricmean*

\begin{proof}
    Applying straightforward computations we take,
    \begin{align*}
            \mathcal{A}(I) & = \|u_i - \ell_i\|_{\infty} = \frac{1}{d}\sum_{i \in [d]}\max_{i \in [d]}\{u_i - \ell_i\} &&& \\
            & \geq \frac{1}{d}\sum_{i \in [d]} u_i - \ell_i && \left[ \text{ equals } \frac{1}{d} \pi(I)\right]  \\
            & \geq \sqrt[d]{\prod_{i \in [d]} u_i - \ell_i} && \left[\begin{array}{l}
            \text{by arithmetic-geometric} \\
            \text{mean inequality. Equals } \sqrt[d]{v(I)}
            \end{array}
            \right]\\
            & \geq \sqrt[d]{\left[\min_{i \in [d]} \{u_i - \ell_i\}\right]^d} && \left[ \text{ equals } \alpha(I))\right]\\
            &= \sqrt[d]{\left[\min_{i \in [d]} (u_i - x_i + x_i - \ell_i)\right]^d}\\
            &\geq \sqrt[d]{\min\left\{\min_{i \in [d]} (u_i - x_i), \min_{j \in [d]} (x_j - \ell_j)\right\}^d} && \left[\begin{array}{l}
            \text{due to non-negativity.} \\
            \text{Equals } \varpi(I) 
            \end{array}\right]\\
    \end{align*}
    \hfill $\blacksquare$
\end{proof}

\subsection{Proofs of Section \ref{sec:robustness}}

In some of the proofs below, we assume w.l.o.g. that the input point is the origin, i.e., $\mathbf{x} = \mathbf{0}$. If $\mathbf{x} \neq \mathbf{0}$, the same arguments apply after translating the space by $- \mathbf{x}$, by working in $\mathbb{F} - \mathbf{x}$.

\bottranslation*

\begin{proof}

W.l.o.g. we assume $\mathbf{x} = \mathbf{0}$. Let $I_j$ be the sequence of intervals produced by Alg. \ref{algo:generic-traversal}, 
    under property $\mathscr{V}_{I, \nu}(\cdot)$ and operator $/$. It holds the following,
    \begin{equation}
        \label{eq:interval-seq-bot-up}
        \left.
        \begin{array}{l l l}
             I_{j+1} &= I_j / \mathbf{v}_{j+1},  &\text{ if }~ \exists \; \mathbf{v}_{j+1} \in I_j \setminus \mathcal{V}\\
             I_{j+1} &= I_j, &\text{ otherwise }
        \end{array}
        \right\}.
    \end{equation}
    From eq. \eqref{eq:interval-seq-bot-up}, $I_{j+1} \subseteq I_i$, for $j = 0, \ldots, n - 1$, and finally $I = I_n$.

Now for a point $\mathbf{x} \in \mathbb{F}$ assume that $\mathbf{x} \in I_j$ for some $j$. If no constraint is applied, $I_{j+1} = I_j$, so $\mathbf{x} \in I_{j+1}$. Otherwise, $I_{j+1} = I_j / \mathbf{v}_{j+1}$. By Def. \ref{def:bot-operator} $\mathbf{x} \in I_j / \mathbf{v}_{j+1}$, so $\mathbf{x} \in I_{j+1}$. Hence, it holds $\mathbf{x} \in I_n = I$.

Again from eq. \eqref{eq:interval-seq-bot-up}, since the constrain operator only removes points, $I_{j+1} \subseteq I_{j}$. Therefore, for all $i \leq j$, $v_i \not\in I_j$, implies $v_i \not\in I_{j+1}$. By induction to the number of adversarial points in $\mathcal{V}$ we take $v \not\in I$, for all $v \in \mathcal{V}$, hence $I \cap \mathcal{V} = \emptyset$.

From Def. \ref{def:bot-operator}, for every coordinate $i$, $-\underline{r}_i < v_i - x_i < \overline{r}_i$ for every $\mathbf{v} \in \mathcal{V}$. Hence no adversarial example lies inside the interval s.t. $(\mathbf{x} + [-\underline{\mathbf{r}}, \overline{\mathbf{r}}]) \cup \mathcal{V} = \emptyset$. Now consider the construction of $I$ from eq. \eqref{eq:interval-seq-bot-up} and assume that at some step $j$ it holds $\mathcal{V}\bot\mathbf{x} \subseteq I_j$. Applying the next constraint and we take $I_{j+1} = I_j / \mathbf{v}_{j+1}$. 

Let $k = \text{arg}\max_i |v_{j+1,i} - x_i|$. If $v_{j+1,i} > x_i$, the new upper bound becomes $u'_k = \max\{0, v_{j+1,i} - x_i - \delta\}$ and by definition of $\overline{r}_k$ we have, $\overline{r}_k \leq v_{j+1,i} - x_i - \delta$. Therefore, $x_k + \overline{r}_k \leq x_k + u'_k$. Since all other coordinates remain unchanged we take $\mathcal{V}\bot\mathbf{x} \subseteq I_{j+1}$. By induction over all adversarial examples, $\mathcal{V}\bot\mathbf{x} \subseteq I$. If $v_{j+1,i} < x_i$ the new lower bound becomes $l'_k = \min\{0, v_{j+1,i} - x_i + \delta\}$. We work analogously as in the previous case and we conclude that $\mathcal{V}\bot\mathbf{x} \subseteq I_{j+1}$.\hfill $\blacksquare$
\end{proof}

\topoperator*

\begin{proof}
    The first desideratum results immediately from Def. \ref{def:top-operator}.
    
    For the second desideratum, we will prove that if $J \subset (\mathcal{Q}\top\mathbf{x})$, then there is some $\mathbf{q} \in Q$, s.t. $\mathbf{q} \notin J$. Let $\mathcal{Q}\top\mathbf{x} = \mathbf{x} + [-\underline{\mathbf{r}}, \overline{\mathbf{r}}]$, with $\underline{\mathbf{r}}, \overline{\mathbf{r}} \geq \mathbf{0}$ as in Def. \ref{def:top-operator}. Consider the $k$-th coordinate of the upper endpoint $\overline{r_i}$. From Def. \ref{def:top-operator} there is some $\mathbf{q} \in \mathcal{Q}$, s.t. $q_k = \overline{r_k}$. Since, $\mathbf{x} \in J$, we write $J = \mathbf{x} + [-\underline{\mathbf{d}}, \overline{\mathbf{d}}]$, with $\underline{\mathbf{d}}, \overline{\mathbf{d}} \geq \mathbf{0}$. Since, $J \subset (\mathcal{Q}\top\mathbf{x})$, then $[-\underline{\mathbf{d}}, \overline{\mathbf{d}}] \subset [-\underline{\mathbf{r}}, \overline{\mathbf{r}}]$. Thus, there is a coordinate of the upper of lower endpoint of $J$, that is dominated by the respective coordinate of $\mathcal{Q}\top\mathbf{x}$. W.l.o.g. let $\overline{d_k} < \overline{r_k} = q_k$, thus $\overline{d_k} < q_k$, and $\mathbf{q} \notin J$. A contradiction.

    For the third desideratum, we work as follows. Let $I, J$ be minimal intervals that include $\mathcal{Q} \cup \mathbf{x}$. Then, $I \cap J = I \sqcap J \in \intervalsdxu$. Moreover, $\mathcal{Q} \subseteq I \cap J$. Thus, the intersection $I \cap J$ satisfies the above properties, and is included in both $I, J$. A contradiction, since we assumed that $I, J$ are minimal.

    The fourth desideratum follows from the previous. From \emph{2} we established that $\mathcal{Q}\top\mathbf{x}$ is minimal. From \emph{3} we establish that the minimal is unique. Finally, the apothem is increasing w.r.t. set inclusion. Thus, $\mathcal{Q}\top\mathbf{x}$ is minimum. \hfill $\blacksquare$
\end{proof}

\toptranslation*

\begin{proof}

W.l.o.g., we assume $\mathbf{x} = \mathbf{0}$. First we show that Alg. \ref{algo:generic-traversal}, under property $\mathscr{Q}_{I, \nu}(\cdot)$ and operator $\sqcup$, terminates after a finite number of steps. To that end, let $I_j = [-\bell^j, \bu^j]$, with $\bell^j, \bu^j \geq \mathbf{0}$, $j \in \mathbb{N}$, be the interval in the $j$-th iteration. For the sequence of intervals holds,
\begin{equation}
    \label{eq:interval-seq-top-down}
    \left.
    \begin{array}{l l l}
         I_{j+1} &= I_j \sqcup [\mathbf{q}_{j+1} - \delta\mathbf{1},\mathbf{q}_{j+1} + \delta\mathbf{1}],  &\text{ if }~ \exists \; \mathbf{q}_{j+1} \in \mathbb{F} \setminus I_j\\
         I_{j+1} &= I_j, &\text{ otherwise }
    \end{array}
    \right\}
\end{equation}
From eq. \eqref{eq:interval-seq-top-down} $I_j \subseteq I_{j+1}$, for $j = 0, \ldots, n - 1$, and finally $I = I_n$.

Therefore, $I_0 = \mathbf{x}$ and $I_j = \mathbf{x} + [-\mathbf{a}^{j}, \mathbf{b}^{j}]$, with $\mathbf{a}^{j}, \mathbf{b}^{j} \geq 0$ and $\mathbf{a}^{0}, \mathbf{b}^{0} = \mathbf{0}$. For the positive point $\mathbf{q}_{k+1}$, the minimum interval update provides $a_i^{j+1} = \max\{a_i^{j}, x_i - q_{j+1, i}\}$, and $b_i^{j+1} = \max\{ b_i^{j}, q_{j+1, i} - x_i \}$. After processing all points in $\mathcal{Q}$ we take, $a^n_i = \max_j\{ x_i - q_{j, i} \mid \mathbf{q}_j \in \mathcal{Q}\}$ and $b^n_i = \max_j\{q_{j+1, i} - x_i \mid \mathbf{q}_j \in \mathcal{Q}\}$. So, we have, $I_n = \mathbf{x} + [-\mathbf{a}^{n}, \mathbf{b}^{n}]$. By Def. \ref{def:top-operator}, $\underline{R_i} = \sup\{x_i - q_{i} \mid \mathbf{q} \in \mathcal{Q}\}$ and $\overline{R_i} = \sup\{q_{i} - x_i \mid \mathbf{q} \in \mathcal{Q}\}$. Since $\mathcal{Q}$ is finite, we take $\underline{R_i} = \max\{x_i - q_{i} \mid \mathbf{q} \in \mathcal{Q}\} = \max_j\{x_i - q_{j, i}\} = a_i^n$ and $\overline{R_i} = \max_j\{q_{j, i} - x_i  \mid \mathbf{q} \in \mathcal{Q}\} = b_i^n$. Hence, $I_n = \mathbf{x} + [-\underline{\mathbf{R}}, \overline{\mathbf{R}}] = \mathcal{Q}\top\mathbf{x}$.\hfill $\blacksquare$
\end{proof}


\duality*

\begin{proof}

From \emph{1.} in Thm. \ref{theo:bot-operator} we have $(\mathcal{V} \bot \mathbf{x}) \cap \mathcal{V} = \emptyset$. Since $\mathcal{Q} \cap \mathcal{V} = \emptyset$, it holds either $\mathcal{Q} \subset (\mathcal{V} \bot \mathbf{x}) \cap \mathbb{F}$ or $(\mathcal{V} \bot \mathbf{x}) \cap \mathbb{F} \subseteq \mathcal{Q}$. If the first inclusion holds, then there exists some $\mathbf{x}' \in (\mathcal{V} \bot \mathbf{x}) \cap \mathbb{F}$ s.t. $\mathbf{x}' \not \in \mathcal{Q}$ and $\mathbf{x}' \not \in \mathcal{V}$. This, in turn, yields $\mathcal{Q} \cup \mathcal{V} \subset \mathbb{F}$, a contradiction. Therefore, it holds $(\mathcal{V} \bot \mathbf{x}) \subseteq \mathcal{Q}$. Now, from \emph{1.} in Thm. \ref{theo:top-operator} we have that $(\mathcal{Q} \top \mathbf{x}) \cup \mathcal{Q} = (\mathcal{Q} \top \mathbf{x})$, meaning that $ \mathcal{Q} \subseteq (\mathcal{Q} \top \mathbf{x})$. Combining all together we take $(\mathcal{V} \bot \mathbf{x}) \cap \mathbb{F} \subseteq \mathcal{Q} \top \mathbf{x}$. Finally, since $\mu$ is an increasing measure the inequality $\mu[(\mathcal{V} \bot \mathbf{x}) \cap \mathbb{F}] \leq \mu[\mathcal{Q} \top \mathbf{x}]$ follows.\hfill $\blacksquare$
\end{proof}

\subsection{Proofs of Section \ref{sec:algorithms}}

\genericcomplexity*

\begin{proof}
From lines 2-3 in Alg. \ref{algo:generic-traversal}, that updates the interval $I$, we have complexity $O(t_{\square} + t_{\mathscr{P}_{I, \nu}})$. Then Alg. \ref{algo:generic-traversal} is linear to the number of oracle calls, therefore, is $O(n)$. Hence the total time complexity becomes $O(n \cdot (t_{\square} + t_{\mathscr{P}_{I, \nu}}))$.

\noindent Now, consider the potential function $\Phi \colon\intervalsdzerou\to\mathbb{R}_{\geq 0}$ s.t.,
\begin{equation}\label{eq:potential}
    \Phi([-\bell, \bu]) = \sum_{i \in [d]} \ell_i + u_i, \quad \text{ for } \bell, \bu \geq \mathbf{0}.
\end{equation}
For $\square = \sqcup$, we have $I_0 = [\mathbf{0}, \mathbf{0}]$, and $\Phi(I_0) = 0$. For every $j \in [d]$, the potential is upper-bounded, as proved in Prop. \ref{prop:top-translation} as well, i.e., $\Phi(I_j) \leq \Phi(\mathbb{F})$. Moreover, for every interval properly included in $\mathbb{F}$, the potential function is increasing with $\Phi(I_{j+1}) \leq \Phi(I_{j}) + \delta$ for some $\delta > 0$. Since the potential increases by at least $\delta$ at each iteration while remaining bounded, the sequence generated from eq. \eqref{eq:potential} must converge. Thus, there exists an index $k \in \mathbb{N}$, s.t. $I_k = I_{k+1}$, and the algorithm terminates after a finite number of steps, so $t_{\sqcup} = O(d)$.

For $\square = /$, we have $I_0 = \mathbb{F}$, and $\Phi(I_0) = \Phi(\mathbb{F})$. Moreover, $\Phi(I_j) \geq 0$ for all $j \in \mathbb{N}$ and for every interval not included in $I_j$, the potential is strictly decreasing, that is $\Phi(I_{j+1}) \leq \Phi(I_j) - \delta$ for some $\delta > 0$. Therefore, there exists an index $k \in \mathbb{N}$, s.t. $I_k = I_{k+1}$, and the algorithm terminates after finitely many steps. The number of iterations is bounded by $t_{/} = O(d)$.

The final argument follows plugging into the total complexity of Alg. \ref{algo:generic-traversal} the time complexity of operators $\{/, \sqcup\}$.\hfill $\blacksquare$
\end{proof}

\soundmaxhard*

\begin{proof}
Let $G = (V, E)$ a simple undirected graph with $|V| = d$. We fix any constant $w > 0$ and wlog we work on $[0, 1]^d$. From Lem. 2 in \cite{backer_mono-_2010} if $I \subseteq [0, 1]^d$ is a maximum-volume feasible interval, then $\partial I$ contains $\mathbf{0}$ and the $i$-th dimension of $I$ is either $w$ or $1$. Therefore, the maximum-volume interval $I$ has $\ell_i = 0$ for every $i$, and each $u_i \in \{w, 1\}$, meaning that $I = [\mathbf{0}, \mathbf{u}]$. The interval $I$ has volume $v(I) = \prod_{i \in [d]} (u_i - \ell_i)$. Applying Th. 2 from \cite{backer_mono-_2010}, graph $G$ has an independent set of size at least $k$ iff there exists an interval $I \in \universe$ of volume at least $w^k$. Consequently, deciding whether there exists any empty interval of volume at least $w^k$ is \textbf{NP}-hard. Taking $\gamma = w^k$ and $\mathbf{q} = \mathbf{0}$, every interval of the form $[\mathbf{0}, \mathbf{u}]$ contains $\mathbf{q}$, so deciding the existence of an empty rectangle that contains the query point $\mathbf{q}$ with $v(I) \geq \gamma$ is \textbf{NP}-hard. Therefore, the q-\textsc{MER} decision problem is \textbf{NP}-hard, and the proof is complete. \hfill $\blacksquare$
\end{proof}

\end{document}